\documentclass[journal]{IEEEtran}

\usepackage{graphicx}
\usepackage{cite}
\usepackage{multirow}
\usepackage[cmex10]{amsmath}
\usepackage{hyperref}
\usepackage{bm}
\usepackage{bbm}
\usepackage[cmex10]{amsmath}
\usepackage{amsfonts}
\usepackage{amssymb}
\usepackage{amsthm}

\newtheorem{theorem}{Theorem}

\newtheorem{definition}[theorem]{Definition}

\newcommand{\mat}[1]{\ensuremath{\bm{\mathrm{#1}}}}

\newcommand{\X}{\mathcal{X}}
\newcommand{\Y}{\mathcal{Y}}
\newcommand{\Phim}{\mat{\Phi}}

\newcommand{\xv}{\mat{x}}
\newcommand{\yv}{\mat{y}}
\newcommand{\uv}{\mat{u}}
\newcommand{\vv}{\mat{v}}
\newcommand{\wv}{\mat{w}}
\newcommand{\lv}{\mat{l}}
\newcommand{\pv}{\mat{p}}
\newcommand{\qv}{\mat{q}}
\newcommand{\rv}{\mat{r}}

\newcommand{\Pb}{\mathbb{P}}
\newcommand{\Ex}{\mathbb{E}}

\begin{document}
\title{Binary adaptive embeddings from order statistics of random projections}

\author{Diego~Valsesia,
        and~Enrico~Magli% <-this % stops a space
\thanks{The authors are with Politecnico di Torino, Torino, Italy.This work was  supported  by  the  European  Research  Council  through  the  European Community Seventh Framework Programme (FP7/2007-2013) under Grant 279848.}}% <-this % stops a space

\maketitle

\begin{abstract}
We use some of the largest order statistics of the random projections of a reference signal to construct a binary embedding that is adapted to signals correlated with such signal. The embedding is characterized from the analytical standpoint and shown to provide improved performance on tasks such as classification in a reduced-dimensionality space.
\end{abstract}


% \fi
\IEEEpeerreviewmaketitle

\vspace{-0.1cm}
\section{Introduction}
\vspace{-0.1cm}

The ever-increasing amount of information that is produced in the age of Big Data calls for efficient techniques for storage and processing of a large number of high-dimensional signals. Compact representations can be obtained with different methods depending whether the particular task requires signal reconstruction (e.g., image and video compression for delivery and visualization) or the goal is to infer some information from the signals (e.g., in classification, regression, information retrieval, etc.). Embeddings provide compact representations of signals for the latter tasks. Formally, an embedding is a
transformation that maps a set of signals in a high dimensional space to a lower dimensional space, in  such a way that the geometry of the set is approximately preserved. The concept of embedding has been successfully used in the context  of  information  retrieval \cite{achlioptas2003database}, where it is usually called ``hashing''. An important class of signal embeddings are those preserving the distances among pairs of signals. Johnson and Lindenstrauss \cite{joh84} famously stated that an embedding can be realized with a Lipschitz mapping to approximately preserve Euclidean distances with a dimension of the embedding space that only depends on the desired distortion and logarithmically in the number of signals to be embedded. Random projections have been shown to implement such embedding with high probability. Several  extensions  have later  been  proposed, allowing one to approximately preserve the angle between signals \cite{Charikar2002,jacques2011robust}, control the maximum distance that is embedded \cite{Boufounos2013}, or preserve the Jaccard distance \cite{Broder1997, sparsehash}. Recently, some works have studied learning embeddings \cite{norouzi2012hamming,NIPS2008_3383,He_2013_CVPR} from training data to derive compact codes by exploiting the particular geometry of the dataset (e.g., signals living close to a manifold).

In this paper we propose an approach to construct an embedding that is not based on learning and does not require a training set of data, but rather is adapted to a single reference signal. This choice maintains to some degree the universality of random projections and it is useful when the data present no particular structure. Jegou et al. \cite{queryadaptive} empirically explored a similar idea by proposing to choose hash functions with a robustness criterion, that essentially measures how far a random projection falls from the edges of a quantization interval. Our work presents a rigorous analytical treatment of a binary embedding obtained from the selection of the random projections of a reference signal with largest magnitude. The analysis of the embedding provides insights on its advantages, particularly in mitigating the difficulty of low-contrast nearest neighbor problems and superior performance on classification tasks, e.g. in a neural network.
\vspace{-0.15cm}
\section{Proposed method}
\vspace{-0.15cm}

\subsection{Preliminaries}

\begin{definition}\label{def:embedding}
A mapping $\phi : \X \to \Y$ of metric spaces, endowed with distances $d_{\X}$ and $d_{\Y}$ is called an embedding with distortion $C > 0$ if $ L d_{\X}(\uv,\vv) \leq d_{\Y}(\phi(\uv),\phi(\vv)) \leq CL d_{\X}(\uv,\vv) $ for some constant $L>0$ and for all $\uv, \vv \in \X$.
\end{definition}

\vspace{-4pt}
A well-known binary embedding is the sign random projections \cite{Charikar2002} where a random matrix $\Phim$ made of independent and identically distributed (i.i.d.) Gaussian entries is used to compute some random projections, which are then quantized to a binary representation by keeping their sign.
The Hamming distance between the binary vectors approximately preserves the angle between the signals in the original space \cite{Charikar2002}, i.e., $\Pb\left( \mathrm{sign}(\Phim_i\uv) =  \mathrm{sign}(\Phim_i\vv) \right) = 1 - \frac{\theta}{\pi}$, being $\theta = \cos^{-1}\left( \frac{\uv^T\vv}{\Vert \uv \Vert \Vert \vv \Vert} \right)$ and $\Phim_i$ the $i$-th row of $\Phim$.

\vspace{-12pt}
\subsection{Proposed adaptive embedding}

A reference signal $\uv$ is used to generate the adaptive embedding in the following way. A number $m_\text{pool}$ of random projections is computed by means of an i.i.d. Gaussian matrix $\Phim \in \mathbb{R}^{m_\text{pool} \times n}$ as $\yv = \Phim \uv$.
The $m$ entries with largest magnitude are identified and their locations stored in vector $\lv$.
The $m$-bit resulting binary code is:
\begin{align}
\label{eq:adaptive}
    \pv = \mathrm{sign}(\Phim_{\lv} \uv)~,
\end{align}
where $\Phim_{\lv}$ is the matrix $\Phim$ restricted to the rows indexed by $\lv$. The locations vector $\lv$ is saved as side information of the embedding and used as in \eqref{eq:adaptive} whenever a new signal is to be embedded, i.e., $\qv=\mathrm{sign}(\Phi_{\lv} \vv)$, where $\vv$ is a generic signal and $\qv$ its embedding.

The first theorem we prove confirms that the Hamming distance $d_H(\pv,\qv)  = \frac{1}{m} \sum_i p_i \oplus q_i$, i.e. the number of differing entries, between binary codes obtained with the adaptive embedding concentrates around its expected value.

\begin{theorem}\label{thm:expected}
Let $\X \subset \mathbb{R}^n$ be a set of $N$ signals and $\uv, \vv \in \X$. Let $\Phim \in \mathbb{R}^{m_\text{pool} \times n}$ with $\Phim_{i,j} \sim \mathcal{N}(0,\sigma^2)$, $\yv = \Phim \uv$ and  the locations $\lv$ of the $m \leq m_\text{pool}$ entries in $\yv$ with largest magnitude be known. Let $\pv = \mathrm{sign}(\Phim_{\lv} \uv)$ and $\qv = \mathrm{sign}(\Phim_{\lv} \vv)$. Then for $ 0 < \varepsilon \leq 2e-1 $,
\begin{align}
\label{eq:concentration}
\Pb\left( \vert d_H(\pv,\qv) - \mu \vert > \epsilon \mu \right) < e^{-\mu \frac{\varepsilon^2}{2}} + e^{-\mu \frac{\varepsilon^2}{4}}
\end{align}
\vspace{-8pt}
with
\vspace{-2pt}
\begin{align*}
\mu &= \Ex\left[ d_H(\pv,\qv) \right] = \frac{1}{m}\sum_{i=1}^m p_i \\
p_i &= \frac{1}{2} + \frac{1}{2}\mathrm{erf}\left( -y_{l_i} \frac{\uv^T \vv}{\sqrt{2}\sigma\Vert \uv \Vert \sqrt{\Vert \uv \Vert^2 \Vert \vv \Vert^2 - (\uv^T \vv)^2}} \right) 
\end{align*}
\end{theorem}

\begin{proof}
Let us consider a single measurement in the $l_i$ location $y_i=\Phim_{l_i} \uv$ and $z_i=\Phim_{l_i} \vv$. Then $\mathbf{\zeta} = [y_i,z_i]$ is a bivariate Gaussian with zero mean and covariance $\Sigma = \sigma^2 \begin{bmatrix}
\Vert \uv \Vert^2 & \uv^T\vv \\ 
\uv^T\vv & \Vert \vv \Vert^2.
\end{bmatrix}$.
Suppose that $y_i$ is observed to be $y_i=\tau_i$, then the conditional distribution of $z_i$ given $y_i=\tau_i$ is $\eta_{\tau_i} = (z_i \vert y_i=\tau_i) \sim \mathcal{N}\left( \tau_i \frac{\uv^T\vv}{\Vert \uv \Vert^2}, \sigma^2 \left( \Vert \vv \Vert^2 - \frac{(\uv^T\vv)^2}{\Vert \uv \Vert^2} \right) \right)$ .

After quantization of the measurements, the probability of mismatching bits in position $l_i$ is 
\vspace{-4pt}
\begin{align*}
p_i &= \Pb\left( \eta_{\tau} \leq 0 \vert \tau_i>0 \right) \\
&= \frac{1}{2} + \frac{1}{2}\mathrm{erf}\left( -\tau_i \frac{\uv^T \vv}{\sqrt{2}\sigma\Vert \uv \Vert \sqrt{\Vert \uv \Vert^2 \Vert \vv \Vert^2 - (\uv^T \vv)^2}} \right)
\end{align*}

Define the following random variable
\begin{align}\label{eq:bit}
    E_i = \begin{cases}
    0 &\text{with probability } 1-p_i \\
    1 &\text{with probability } p_i \end{cases}.
\end{align}

Then, $D_H = \frac{1}{m}\sum_{i=1}^m E_i$ is a Poisson Binomial random variable measuring the Hamming distance between $\pv$ and $\qv$. Eq. \eqref{eq:concentration} can be readily obtained using Chernoff bounds \cite{mitzenmacher2005probability} for the tails of $D_H$. 
\end{proof}

The previous theorem holds for a fixed pair of signals. It is customary to derive an asymptotic result on the number of measurements needed to provide a distortion $\delta$ around the expectation when signals are drawn from a finite set of cardinality $N$. Standard derivation using a union bound on Eq. \eqref{eq:concentration} yields $m = \mathcal{O}(\delta^{-2}\log N)$, which is exactly the same as classic results on non-adaptive random projections \cite{joh84}. The advantages of the proposed method are, in fact, due to the modified expected value rather than the variance.

Moreover, the previous theorem supposed we knew the values of the projections of the reference signal at the locations kept as side information. This allows us to compute the exact expected value of the Hamming distance in the embedded space as function of the inner product (or correlation coefficient) between the original signals. However, it might be useful to have some a-priori knowledge about the embedding without the need to know the reference signal. The following theorem approximately bounds the expected value of the embedding by characterizing the order statistics of $\vert \yv \vert$. The $k$-th order statistic of a statistical sample is equal to its $k$th-smallest value.

Let us call $f_{k}(\tau)$ the probability density function of the $k$-th order statistic of $\vert \yv \vert$. We could then in principle compute the probability of bit mismatch as
\begin{align*}
p_i = \int \Pb\left( \eta_{\tau} \leq 0 \vert \tau \right) f_{m_{\text{pool}}-i+1}(\tau) d\tau \quad \text{for } i=1, 2, \dots, m  
\end{align*}
and then repeat the same Poisson binomial argument as before. However, this is cumbersome to compute and we instead derive some bounds.

\vspace*{-4pt}
\begin{theorem}
Under the same assumptions as Theorem 1, and being $e_{2(m_\text{pool}-m+1);2m_\text{pool}}$ the expected value of the $2(m_\text{pool}-m+1)$-th order statistic of a sample of $\mathcal{N}(0,\sigma^2\Vert u \Vert^2)$ of size $2m_\text{pool}$:
\vspace*{-6pt}
\begin{align*}
\Ex\left[ d_H(\pv,\qv) \right] \leq \frac{1}{2}+ \frac{1}{2}\mathrm{erf}\left( \frac{-e_{2(m_\text{pool}-m+1);2m_\text{pool}}  \uv^T \vv}{\sqrt{2}\sigma\Vert \uv \Vert \sqrt{\Vert \uv \Vert^2 \Vert \vv \Vert^2 - (\uv^T \vv)^2}} \right)
\end{align*}
\end{theorem}

\begin{proof}
We first notice that $p_i \leq p_{m}$, $\forall i=1,\dots,m$ so that $\Ex\left[ d_H(\pv,\qv) \right] = \frac{1}{m}\sum_{i=1}^m p_i \leq p_{m}$. Then, $p_m = \Ex_{\tau}\left[ g(\tau) \right] \leq g( \Ex[\tau] )$
by applying Jensen's inequality to $g(\tau) = \Pb\left( \eta_{\tau} \leq 0 \vert \tau \right)$, i.e., the same Gaussian tail probability as before. Also notice that the convexity of $g$ allows us to use Jensen's inequality.
$\Ex[\tau] = \tilde{e}_{(m_\text{pool}-m+1);m_\text{pool}}$ is the expected value of the $(m_\text{pool}-m+1)$-th order statistic of a sample of size $m_\text{pool}$ from a half Gaussian (since we consider $\vert \yv \vert$).
We then notice that that is equivalent \cite{balakrishnan2014order} to $e_{2(m_\text{pool}-m+1);2m_\text{pool}}$, i.e., the $2(m_\text{pool}-m+1)$-th order statistic of a sample of size $2m_\text{pool}$ of a full Gaussian with zero mean and $\sigma^2\Vert u \Vert^2$ variance.
\end{proof}

As a further remark, according to \cite{harter_orderstatistics} an approximation of the expected value of the desired order statistic is:
\begin{align*}
    e_{2(m_\text{pool}-m+1);2m_\text{pool}} \approx F^{-1}\left( \frac{2(m_\text{pool}-m+1)-\alpha}{2m_\text{pool}-2\alpha+1} \right)
\end{align*}
being $\alpha=0.375$ and $F^{-1}$ the inverse CDF of a normal distribution with zero mean and variance $\sigma^2\Vert u \Vert^2$.

So far we considered distances between a test signal and the reference used to adapt the embedding. We will now consider what happens to the distance between any arbitrary pair of signals $\vv$ and $\wv$. Qualitatively, we can say that the curve of the expected value of the Hamming distance in the embedded space as function of the original distance between $\vv$ and $\wv$ will be somewhere between the one predicted by Theorem \ref{thm:expected} and the one of non-adaptive sign random projections depending on how much $\wv$ is close to the reference $\uv$. The following theorem formalizes this concept.

\begin{theorem}\label{thm:expected_3party}
Let $\X \subset \mathbb{R}^n$ be a set of $N$ signals and $\uv, \vv, \wv \in \X$. Let $\Phim \in \mathbb{R}^{m_\text{pool} \times n}$ with $\Phim_{i,j} \sim \mathcal{N}(0,\sigma^2)$, $\yv = \Phim \uv$ and  the locations $\lv$ of the $m \leq m_\text{pool}$ entries in $\yv$ with largest magnitude be known. Let $\pv = \mathrm{sign}(\Phim_{\lv} \uv)$, $\qv = \mathrm{sign}(\Phim_{\lv} \vv)$, and $\rv = \mathrm{sign}(\Phim_{\lv} \wv)$. Then for $ 0 < \varepsilon \leq 2e-1 $,
\vspace*{-4pt}
\begin{align}\label{eq:concentration_thm4}
\Pb\left( \vert d_H(\qv,\rv) - \mu \vert > \epsilon \mu \right) < e^{-\mu \frac{\varepsilon^2}{2}} + e^{-\mu \frac{\varepsilon^2}{4}}
\end{align}
\vspace*{-6pt}
with
\vspace*{-8pt}
\begin{align*}
\mu &= \Ex\left[ d_H(\qv,\rv) \right] = \frac{1}{m}\sum_{i=1}^m p_i \\
p_i &= \left[ F\left( \begin{bmatrix}
0\\ 
+\infty
\end{bmatrix} \right) - F\left( \begin{bmatrix}
0\\ 
0
\end{bmatrix} \right) \right] \left[ 1 - F\left( \begin{bmatrix}
+\infty\\ 
0
\end{bmatrix} \right) \right] \nonumber \\
&+ \left[ F\left( \begin{bmatrix}
+\infty\\ 
0
\end{bmatrix} \right) - F\left( \begin{bmatrix}
0\\ 
0
\end{bmatrix} \right) \right] F\left( \begin{bmatrix}
+\infty\\ 
0
\end{bmatrix} \right),
\end{align*}
being $F$ the CDF of a bivariate Gaussian with mean $\mathbf{\mu}' = \frac{y_{l_i}}{\Vert \uv \Vert^2} \begin{bmatrix}
\uv^T \vv\\ 
\uv^T \wv
\end{bmatrix}$ and covariance $\mathbf{\Sigma}' = \sigma^2 \begin{bmatrix}
\Vert \vv \Vert^2 - \frac{(\uv^T \vv)^2}{\Vert \uv \Vert^2}  & \vv^T \wv - \frac{(\uv^T \vv)(\uv^T \wv)}{\Vert \uv \Vert^2}  \\ 
\vv^T \wv - \frac{(\uv^T \vv)(\uv^T \wv)}{\Vert \uv \Vert^2}  & \Vert \wv \Vert^2 - \frac{(\uv^T \wv)^2}{\Vert \uv \Vert^2}
\end{bmatrix}$.
\end{theorem}

\begin{proof}
The proof is similar to the proof of Theorem \ref{thm:expected}. Let us consider a single measurement in the $l_i$ location $y_i=\Phim_{l_i} \uv$, $z_i=\Phim_{l_i} \vv$, $a_i=\Phim_{l_i} \wv $. Then $\mathbf{\zeta} = [z_i,a_i,y_i]$ is Gaussian with zero mean and covariance $\Sigma = \sigma^2 \begin{bmatrix}
\Vert \uv \Vert^2 & \uv^T\vv & \uv^T\wv \\ 
\uv^T\vv &  \Vert \vv \Vert^2 & \vv^T\wv \\
\uv^T\wv & \vv^T\wv &  \Vert \wv \Vert^2.
\end{bmatrix}$.
Suppose that $y_i$ is observed to be $y_i=\tau_i$, then the conditional distribution of $[z_i,a_i]$ given $y_i=\tau_i$ is $\mathbf{\eta} = ([z_i,a_i] \vert y_i=\tau_i) \sim \mathcal{N}\left( \mathbf{\mu}', \mathbf{\Sigma}' \right)$ with $\mathbf{\mu}' = \frac{y_{l_i}}{\Vert \uv \Vert^2} \begin{bmatrix}
\uv^T \vv\\ 
\uv^T \wv
\end{bmatrix}$ and covariance $\mathbf{\Sigma}' = \sigma^2 \begin{bmatrix}
\Vert \vv \Vert^2 - \frac{(\uv^T \vv)^2}{\Vert \uv \Vert^2}  & \vv^T \wv - \frac{(\uv^T \vv)(\uv^T \wv)}{\Vert \uv \Vert^2}  \\ 
\vv^T \wv - \frac{(\uv^T \vv)(\uv^T \wv)}{\Vert \uv \Vert^2}  & \Vert \wv \Vert^2 - \frac{(\uv^T \wv)^2}{\Vert \uv \Vert^2}
\end{bmatrix}$.

After quantization of the measurements, the probability of mismatching bits in position $l_i$ is 
\vspace{-4pt}
\begin{align*}
p_i &= \Pb\left( \eta_1 < 0 \vert \eta_2 > 0 \right) \Pb\left( \eta_2 > 0 \right)  +  \Pb\left( \eta_1 > 0 \vert \eta_2 < 0 \right)  \Pb\left(\eta_2 < 0 \right)
\vspace{-8pt}
\end{align*}

Define the random variable $E_i$ as in \eqref{eq:bit}, then $D_H = \frac{1}{m}\sum_{i=1}^m E_i$ is a Poisson Binomial random variable measuring the Hamming distance between $\qv$ and $\rv$. Eq. \eqref{eq:concentration_thm4} can be readily obtained using Chernoff bounds \cite{mitzenmacher2005probability} for the tails of $D_H$. 
\end{proof}

The key distinction between sign random projections \cite{Charikar2002} and the method presented in this paper is that the former provides a linear relationship between the Hamming distance in the embedded space and the angle in the original space. On the other hand, the binary adaptive embedding provides a nonlinear relationship between the two distances, with the important property that the Hamming distances observed with the adaptive embedding are always smaller than those observed with sign random projections. This property is at the core of the improved performance of the embedding for tasks such as binary classification, as discussed in Sec. \ref{sec:experimental}. Fig. \ref{fig:2party} shows the Hamming distance in the embedded space as function of the inner product between the test signal and the reference signal in the original space. It can be noticed how the curve of the adaptive embedding always lies below the one for the non-adaptive embedding. For a fixed value of $m_\text{pool}$ increasing the umber of measurements $m$ will reduce the variance and move the expected value towards that of the non-adaptive embedding. Viceversa, for a fixed $m$ increasing $m_\text{pool}$ will lower the curve. Fig. \ref{fig:3party} shows the Hamming distance between the first test signal and the second test signal in the embedded space as function of the inner product between the first test signal and the second test signal ($\vv^T\wv$) and between the second test signal and the reference signal ($\uv^T\wv$) in the original space. Notice how for decreasing $\uv^T\wv$ the shape of the embedding tends to the one of a non-adaptive embedding.

\begin{figure}
    \centering
    \vspace*{-0.3cm}
    \includegraphics[width=0.85\columnwidth]{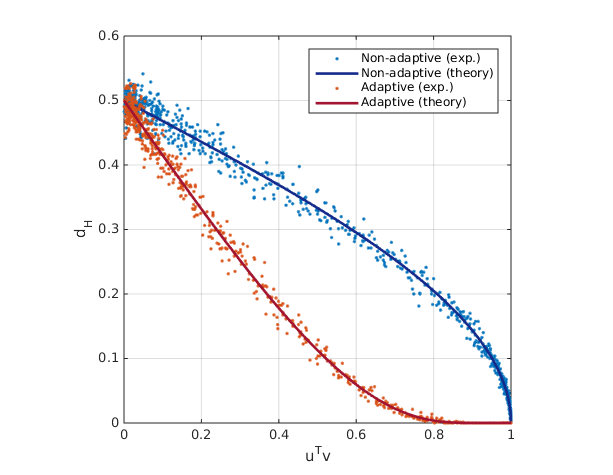}
    \vspace*{-0.5cm}
    \caption{Hamming distance in the embedding a function of inner product between reference and test signals. Unit norm signals, $m=800$, $m_\text{pool}=5000$. Non-adaptive curve uses sign random projections \cite{Charikar2002}.}
    \vspace*{-0.46cm}
    \label{fig:2party}
\end{figure}

\begin{figure}
    \centering
    \includegraphics[width=0.8\columnwidth]{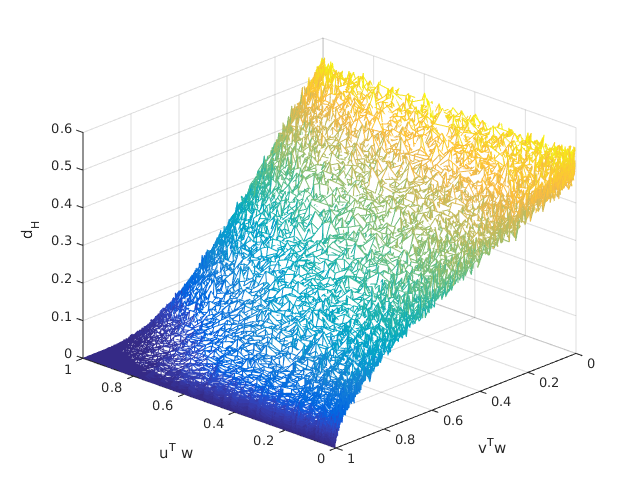}
    \vspace*{-0.75cm}
    \caption{Hamming distance in the embedding a function of inner product between reference and test signals and between both test signals. Unit norm signals, $m=800$, $m_\text{pool}=5000$.}
    \vspace*{-0.5cm}
    \label{fig:3party}
\end{figure}

\section{Applications and experiments}
\label{sec:experimental}

\subsection{Low-contrast approximate nearest neighbors}
A measure of difficulty \cite{NNcontrast, Beyer1999} of a nearest neighbor search problem is the contrast $r$, which is defined as the ratio between the distance of the closest false neighbor and that of the farthest true neighbor. Locality sensitive hashing \cite{AndoniLSH2008} is able to find approximate nearest neighbors in a time $\mathcal{O}(N^r)$ that is sublinear in the database size $N$ but that degenerates to linear search as the contrast approaches 1. The curse of dimensionality makes low contrast more probable when high-dimensional spaces are considered. The adaptive embedding presented in this paper can be used as a dimensionality reduction technique to solve approximate nearest neighbor search more efficiently in low-contrast scenarios.

In order to show this we develop an experiment in a high-dimensional space where the goal is to find the nearest neighbors of a given signal within a certain radius. True neighbors are generated as standard Gaussian vectors with $n=8192$ i.i.d. entries with an expected  correlation coefficient equal to 0.07 to a reference that is also used as query. Disturbing signals are i.i.d. Gaussian with zero expected correlation. Notice that they are almost orthogonal to each other but the contrast is low because the true neighbours are weakly correlated with our query. This problem is not unrealistic and, in fact, as an example, it occurs in the detection of photo-response non-uniformity artifacts from camera sensors \cite{Toothpic_TIFS}\cite{Toothpic_TMM}, used to attribute a given picture to a given camera sensor. In this experiment all the signals in the database are adaptively embedded, i.e., the locations of the $m$ entries of largest magnitude are identified and stored with the binary code. The random projections of the query are computed and appropriately subsampled according to the locations stored for each database signal under test. The storage requirement for each database entry is the sum of the bits needed by the binary measurements and the overhead due to the adaptively chosen locations and it amounts to $m + \log_2\left[ m_\text{pool} \choose m \right]$ bits.
Fig. \ref{fig:roc} shows the Receiver Operating Characteristic (ROC) showing the probability of detection of a true neighbor against the probability of false alarm. We notice that the adaptive embedding provides a performance closer to the uncompressed case. Two non-adaptive strategies are presented for a fair comparison. A non-adaptive method using the same storage as the adaptive method would use $m'=m + \log_2\left[ m_\text{pool} \choose m \right]$ binary random projections, with the drawback of increased computational complexity in the Hamming distance evaluation. The second strategy equalizes computational complexity, thus using $m'=m$ non-adaptive measurements. This is advantageous in terms of storage but it performs significantly worse. Finally, we compared the proposed method with the Universal Embedding of Boufounos et al. \cite{Boufounos2013} which is a kind of adaptive embedding where the quantizer can be parametrized in order to distort the expected value of signal distances, similarly to the embedding proposed in this paper. The Universal Embedding bounds the maximum distance that is embedded, beyond which points become indistinguishable. It is therefore expected to have poor performance in low-contrast, low-correlation scenarios, as it appears from Fig. \ref{fig:roc}., where the quantization step is $\Delta=2$.

\begin{figure}
    \centering
    \vspace*{-0.05cm}
    \includegraphics[width=0.6\columnwidth]{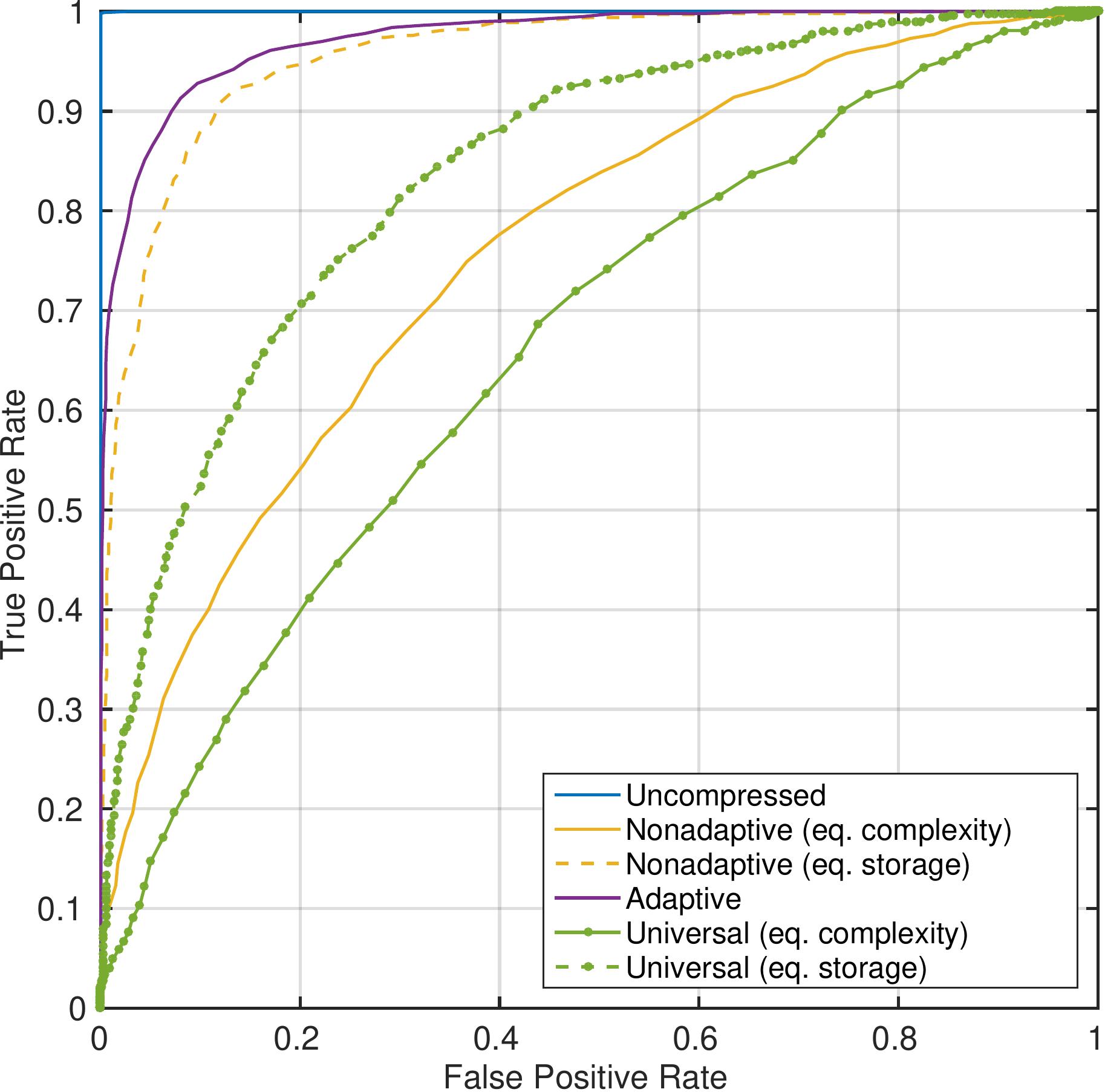}
    \vspace*{-0.2cm}
    \caption{ROC for nearest neighbor search. A point is declared a neighbor of the query if the Hamming distance is below a given threshold. $n=8192$, $m_\text{pool}=8192$, $m=512$. The query is used as reference signal.}
    \vspace*{-0.6cm}
    \label{fig:roc}
\end{figure}

\vspace*{-10pt}
\subsection{Multiclass linear classifiers}
\vspace*{-0.1cm}
In this section we apply the adaptive embedding to a multiclass linear classifier in order to improve its storage and computational efficiency. Linear classifiers are widely used in the context of deep neural networks, where the layers of the network are trained to disentangle the features of each class and a simple linear classification layer provides the class labels. A $k$-class linear classifier can be written as $l = \arg\max_{i=\lbrace 1,\dots, k \rbrace} \wv_i\xv$,
being $l$ the class label, $\wv_i$ the weights vectors and $\xv$ a feature vector.
The weights are learned during the training phase using a suitable loss function such as the hinge loss \cite{cortes1995support} for support vector machines or the softmax cross-entropy \cite{bishop2006pattern} for multinomial logistic regression which is more popular in deep neural networks.
Since the feature vectors may be high dimensional and the number of classes large, this operation may require significant storage space for the real-valued weights as well as computational resources to compute all the inner products. This can be overcome using an embedding such as sign random projections. After the training phase is completed, the weights are embedded in a compact binary code for each class. During predictions the feature vectors are also embedded, the Hamming distance with the weights is computed and the class label corresponding to the minimum distance is selected, i.e. $ l = \arg\min_{i=\lbrace 1,\dots, k \rbrace} d_H(\mathbf{\omega}_i, \yv)$, being $\mathbf{\omega}_i= \mathrm{sign}\left(\Phi \wv \right)$ and $\yv = \mathrm{sign}\left(\Phi \xv \right)$.
Replacing sign random projections with the proposed adaptive embedding can improve the classification performance of the compressed system. Overall, there are as many adaptive embeddings as the number of classes. The random projections of each $\wv_i$ are used to compute a different set of locations $\lv_i$ for each class. At test time, the feature vector is embedded $k$ times to generate $\yv_i$ (this amounts to generating $m_\text{pool}$ non-adaptive projections and then subsampling according to the corresponding $\lv_i$). Hence, the class label is given by $l = \arg\min_{i=\lbrace 1,\dots, k \rbrace} d_H(\mathbf{\omega}_i, \yv_i)$, being $\mathbf{\omega}_i= \mathrm{sign}\left(\Phi_{\lv_i} \wv \right)$ and $\yv_i = \mathrm{sign}\left(\Phi_{\lv_i} \xv \right)$.

The following experiment is a classification problem on the CIFAR-10 dataset \cite{krizhevsky2009learning} comprising 10 classes. We implemented the same convolutional neural network architecture presented in \cite{simonyan2014very}. This network is composed of 8 convolutional layers followed by 2 fully connected layers all with ReLU activation units \cite{nair2010rectified} and a final linear layer. The last linear layer outputs one of the $k=10$ class labels from a $n=1024$-dimensional input feature vector. After conventional training of the network, we replaced the layer weights with its embedded codes as explained above. For the adaptive method we used $m_\text{pool}=1024$. Table \ref{table:cnn} shows the classification accuracy as function of the number of measurements used by the embedding. It can be noticed that the adaptive embedding allows to achieve a significant dimensionality reduction at a negligible loss in terms of classification accuracy, with respect to both sign random projections \cite{Charikar2002} and the universal embedding \cite{Boufounos2013}. The quantization step size of the universal embedding has been optimized via cross validation to value $\Delta=180$.

\begin{table}
\centering
\caption{Classification accuracy}
\label{table:cnn}
\begin{tabular}{c|c|c|c|c}
\textbf{Method} & \multicolumn{4}{c}{\textbf{Accuracy}}\\
\hline
Uncompressed & \multicolumn{4}{c}{$\textbf{92.64 \%}$}\\
\hline
& $m=32$ & $m=64$ & $m=128$ & $m=256$ \\[0.05cm]
\textbf{Adaptive} & \textbf{92.30\%} & \textbf{92.33\%} & \textbf{92.40\%} & \textbf{92.49\%}\\[0.05cm]
Sign random projections & \multirow{2}{*}{75.49\%} & \multirow{2}{*}{87.03\%} & \multirow{2}{*}{91.09\%} & \multirow{2}{*}{92.27\%} \\
(eq. complexity) & & & & \\
Universal Embedding & \multirow{2}{*}{77.13\%} & \multirow{2}{*}{87.81\%} & \multirow{2}{*}{92.10\%} & \multirow{2}{*}{92.30\%} \\
(eq. complexity) & & & & \\
\hline
& $m=233$ & $m=405$ & $m=680$ & $m=1065$ \\
Sign random projections & \multirow{2}{*}{91.90\%} & \multirow{2}{*}{92.32\%} & \multirow{2}{*}{92.38\%} & \multirow{2}{*}{92.51\%} \\
(eq. storage) & & & &\\
Universal Embedding & \multirow{2}{*}{92.20\%} & \multirow{2}{*}{92.32\%} & \multirow{2}{*}{92.39\%} & \multirow{2}{*}{92.49\%} \\
(eq. storage) & & & & \\
\end{tabular}
\vspace*{-0.5cm}
\end{table}

\vspace*{-0.1cm}
\section{Conclusions}
\vspace*{-0.1cm}
This paper presented a technique to generate compact binary codes from high-dimensional signals adapting them to a reference signal. The resulting embedding displays interesting properties that allow to improve performance in classification tasks when those are performed in the reduced-dimensionality domain. Future work will focus on generalizing the approach to sub-Gaussian and structured sensing matrices.


\end{document}